\documentclass[11pt]{article} 
\usepackage{hyperref}
\usepackage{url}

\usepackage{amsthm,amsmath,amssymb}
\usepackage{graphicx}
\usepackage{algorithm}
\usepackage{algorithmic}

\usepackage{titlesec}
\titleformat{\section}
	{\normalfont\Large\bfseries\filcenter}{\thesection.}{1 ex}{}
\titleformat{\subsection}
	{\normalfont\normalsize\bfseries}{\thesubsection.}{1 ex}{}
\titleformat{\subsubsection}[runin]
	{\normalfont\normalsize\bfseries\filcenter}{\thesubsection.}{1 ex}{}

\usepackage[charter]{mathdesign}
\usepackage[mathcal]{eucal}

\usepackage[margin=1.2 in]{geometry}

\title{Learning with Cross-Kernels and Ideal PCA}

\author{
Franz J.~Kir\'{a}ly\thanks{
Department of Statistical Science,
University College,
Gower Street,
London WC1E 6BT, United Kingdom,
\url{f.kiraly@ucl.ac.uk}}
\and
Martin Kreuzer \thanks{
Faculty of Computer Science and Mathematics,
University of Passau,
94030 Passau, Germany,
\url{martin.kreuzer@uni-passau.de}}
\and
Louis Theran\thanks{
Institut f\"ur Mathematik,
Freie Universit\"at Berlin,
Arnimallee 2,
14195 Berlin, Germany,
\url{theran@math.fu-berlin.de}}
}

\date{}

\theoremstyle{plain}
 \newtheorem{Thm}{Theorem}[section]
 \newtheorem{Prop}[Thm]{Proposition}
 \newtheorem{Cor}[Thm]{Corollary}
 \newtheorem{Lem}[Thm]{Lemma}

\theoremstyle{definition}
 \newtheorem{Rem}[Thm]{Remark}
 \newtheorem{Def}[Thm]{Definition}
 
 \newtheorem{Not}[Thm]{Notation}

\makeatletter
\renewenvironment{proof}[1][\proofname]{\par
  \pushQED{\qed}%
  \normalfont \topsep0\p@\@plus0\p@\relax
  \trivlist
  \item[\hskip\labelsep
        \itshape #1\@addpunct{:}]\ignorespaces
}{%
  \popQED\endtrivlist\@endpefalse
}
\makeatother

\usepackage{graphicx}
\usepackage{subfigure}

\newcommand{\R}{\mathbb{R}}

\newcommand{\calF}{\mathcal{F}}

\newcommand{\calP}{\mathcal{P}}

\newcommand{\calS}{\mathcal{S}}

\newcommand{\calV}{\mathcal{V}}
\newcommand{\calX}{\mathcal{X}}

\newcommand{\calZ}{\mathcal{Z}}

\newcommand{\vt}{\mathbf{t}}

\newcommand{\rk}{\operatorname{rank}}
\newcommand{\frk}{\operatorname{frk}}

\newcommand{\fspan}{\operatorname{fspan}}

\newcommand{\diag}{\mathop{\rm diag}\nolimits}

\newcommand{\Id}{{\rm I}}
\newcommand{\rowspan}{\mathop{\rm rowspan}\nolimits}
\newcommand{\colspan}{\mathop{\rm colspan}\nolimits}

\newcommand{\KDF}{\mathop{\rm KDF}\nolimits}
\newcommand{\eval}{\mathop{\rm eval}\nolimits}

\newcommand{\Kxx}{K_{\!X\!X}}
\newcommand{\Kxz}{K_{\!X\!Z}}
\newcommand{\Kzz}{K_{\!Z\!Z}}
\newcommand{\Kzx}{K_{\!Z\!X}}

\let\phi=\varphi
\let\rho=\varrho
\let\epsilon=\varepsilon

\begin{document}

\maketitle

\begin{abstract}
We describe how cross-kernel matrices, that is, kernel matrices between the data and a custom chosen set of `feature spanning points' can be used for learning. The main potential of cross-kernels lies in the fact that (a) only one side of the matrix scales with the number of data points, and (b) cross-kernels, as opposed to the usual kernel matrices, can be used to certify for the data manifold. Our theoretical framework, which is based on a duality involving the feature space and vanishing ideals, indicates that cross-kernels have the potential to be used for any kind of kernel learning. We present a novel algorithm, Ideal PCA (IPCA), which cross-kernelizes PCA. We demonstrate on real and synthetic data that IPCA allows to (a) obtain PCA-like features faster and (b) to extract novel and empirically validated features certifying for the data manifold.
\end{abstract}

\section{Introduction}
\label{Sec:intro}
Since their invention by Boser, Guyon and Vapnik~\cite{BoserVapnik92,Vapnik95},
kernel methods have had a fundamental impact on the fields of statistics
and machine learning. The major appeal of using kernel methods for learning
consists in using the kernel trick, first proposed by Aizerman, Braverman
and Rozonoer~\cite{Aizerman64}, which allows to make otherwise costly computations
in the feature space implicit and thus highly efficient for a huge variety of
learning tasks -- see e.g.~\cite{Scholkopf02, Shawe-Taylor04} for an overview.

Many kernel methods make extensive use of the so-called kernel matrix, a matrix
whose entries are kernel function evaluations $k(x_i,x_j)$ at data points
$x_1,\dots, x_N\in\R^n$. This kernel matrix $\Kxx=(k(x_i,x_j))_{ij}$ is
simultaneously the main source of efficient linearization and the central
computational bottleneck. For instance, the most expensive part of algorithms
such as kernel PCA or kernel ridge regression consists in computing the inverse
or a singular value decomposition of the $(N\times N)$ matrix $\Kxx$.
We make two observations in this context:

{\bf (A)} Subsampling $\Kxx$, i.e., considering $(N\times M)$ sub-matrices of $\Kxx$,
is the state-of-the-art in speeding up the scalability in the number of data points~$N$.
While much has been written on the topic, there seems to be no consensus on how exactly
to choose the subsample from the data set.

{\bf (B)} There seems to be no kernel algorithm which learns the data manifold
from which the $x_i$ were obtained, or test whether an unseen data point is on the manifold - in particular none involving $\Kxx$.
(Note that kernel PCA outputs projections on predominant data coordinates, but \emph{not} the coordinates embedded in data space or the shape of the data manifold.)

In this paper, we propose an algebraic duality framework which offers both a potential
explanation and a potential solution to the above (seemingly unrelated) issues.
Our contribution is centered around \emph{cross-kernel matrices} $\Kxz=(k(x_i,z_j))_{ij}$,
where $z_1,\dots, z_M\in \R^n$, called \emph{cross-kernel basis points}, are not
necessarily data points. Assuming that $z_1,\dots, z_M$ are chosen suitably (e.g.,
randomly) and that~$M$ is large enough, we show theoretically and empirically that:

{\bf (1)} The cross-kernel matrix $\Kxz$ can be used to capture both the information
contained in $\Kxx$ as well as additional information about the data manifold
which is provably not contained in~$\Kxx$. Practically, we present an algorithm,
called IPCA, which is able to learn both \emph{coordinates in the manifold}
as well as \emph{certificates for being contained in the data manifold},
addressing issue~(B) above.

{\bf (2)} The cross-kernel matrix $\Kxz$ can be employed in ways completely analogous
to a subsampled matrix~$\Kxx$. As opposed to subsampling, the properties of the points $z_i$
can be prescribed by the experimenter and are thus independent of known or unknown -- and
potentially detrimental -- properties of the data~$x_i$. In particular, $\Kxz$ can be used
for speeding up kernel learning, while potentially avoiding data-related subsampling
issues, thereby addressing issue~(A) above.

In the following we briefly explain why and how considering the cross-kernel $\Kxz$
can be advantageous as compared to the kernel matrix $\Kxx$. More detailed technical
statements and explicit algorithms can be found in section~\ref{Sec:duality} and thereafter.

{\bf Why the kernel matrix is not enough.}
The kernel matrix $\Kxx$ misses information on the data manifold that can be obtained from~$\Kxz$.
Let us explain this in the example where the kernel is the ordinary Euclidean scalar product
$k(x,y)=\langle x,y\rangle$ in~$\R^n$.  Suppose the data points $x_1,\dots, x_N$ all lie
in a vector subspace $L\subseteq \R^n$. The entries of the kernel matrix~$\Kxx$ are scalar
products of the type $\langle x_i,x_j\rangle$ which are invariant under rotation of the
coordinate system. Therefore, one obtains one and the same kernel matrix $\Kxx$ when considering rotated data points $Ux_1,\dots, Ux_N$ and the rotated vector space $U\cdot L$ for any rotation matrix $U\in\R^{n\times n}$. Since one and the same kernel matrix $\Kxx$ can arise in a non-degenerate way from every vector space~$L$ (of that dimension), the data manifold~$L$ can not be obtained
back from~$\Kxx$. Similar considerations hold for arbitrary kernels: \emph{the data manifold can not
be learnt from the kernel matrix~$\Kxx$}.

{\bf Cross-kernels help!}
In the above example the vector space~$L$ \emph{can} be identified from a cross-kernel
matrix $\Kxz = (k(x_i,z_j))_{ij}$, if $z_1,\dots, z_M$ span~$\R^n$. Observe that
the rows of~$\Kxz$ are coordinate representations of the~$x_i$. Since $z_1,\dots,z_M$
span~$\R^n$, the data point~$x_i$ can be obtained back from the $i$-th row of $\Kxz$.
Since reconstructing $x_i$ from the $i$-th row of $\Kxz$ is a linear
reparameterization, an unseen data point~$x$ is contained in the data manifold~$L$
if and only if the vector $\kappa_z(x) = (k(x,z_j))_j$ is contained in the row-span of~$\Kxz$. Therefore we can obtain~$L$ back from $\Kxz$ and use the row-span of~$\Kxz$ to efficiently
test membership for the data manifold~$L$. Note that the same reasoning does not work
for the matrix $\Kxx$, since $x_1,\dots, x_N$ do \emph{not} span $\R^n$, but only~$L$.
A similar reasoning holds in larger generality: \emph{the
cross-kernel $\Kxz$ contains extensive information on the data manifold which $\Kxx$ does not.}

{\bf From cross-kernels to kernel learning.}
The ``manifold awareness'' of the cross-kernel implies further interesting properties
which can be used for decomposing, and thus approximating, the original kernel matrix~$\Kxx$
via the potentially much smaller~$\Kxz$. Namely, letting $\Kzx=\Kxz^\top$ and
$\Kzz=(k(z_i,z_j))_{ij}$, we can show that the equality $\Kxx=\Kxz\Kzz^{-1}\Kzx$
holds \emph{exactly} for polynomial kernels (assuming ``feature-spanning'' points $z_i$
as in the example above, see Theorem~\ref{Thm:matcompeq}). We conjecture (and have observed empirically) that the same holds approximately for other kernels. This equality is
reminiscent of the central equation in subsampling, with the difference that the points $z_j$
do not need to be chosen among the data. Suitable choices of the points $z_j$ avoid
potential problems in the data while leading to the same speed-up of arbitrary kernel methods.

\section{Some dualities for polynomial kernels}
\label{Sec:duality}

The following preview of our approach is a more technically exact and quantitative variant
of the discussion in the introduction. In what follows, we use inhomogenous
polynomial kernels. (A further difficulty, not considered here, lies in the generalization
to other kernels and kernel feature spaces.)
The central observation which we to carry over from the introduction
is the certifying property of the row-span of~$\Kxz$.
In our setup, the data points $x_i$ are sampled from a manifold
$\calX\subset \R^n$ which is ``cut out'' by an unknown set of polynomial equations.
Our goal is to relate these equations to the linear span of the images $\Phi_{\le d}(x_i)$
under the feature map $\Phi_{\le d}$. This linear span is then
related to the row-span of~$\Kxz$ by applying several dualities in and on feature space.

More technically, we first show how to identify $\R[\mathbf{t}]_{\le d}$, the set of
polynomials of degree at most~$d$, with linear functionals on the kernel feature space
and that every $f\in \R[\mathbf{t}]_{\le d}$ can be expressed as a kernel decision function.
To understand the equations cutting out $\calX$ (which are not canonical)
it is better to change perspective and consider the set of \emph{all} polynomials
vanishing on~$\calX$, called its \emph{vanishing ideal} (which is canonical).  The main
results of this section say, informally, that $\calX$ can be identified with the
intersection of a particular linear space $L$ and $\Phi_{\le d}(\R^n)$, and that the
vanishing ideal of~$\calX$ corresponds to~$L^\perp$, which in turn is isomorphic
to the row-span of~$\Kxz$. In other words, we relate
a linear duality to the algebraic-geometric concept of duality between
sets of equations and their common vanishing locus.

\subsection{Duality of polynomial rings and feature spaces}
To begin with, we introduce inhomogeneous polynomial kernel as follows.

\begin{Def}
Let $\theta \in (0,1)$ be a fixed real number, and let $d\ge 1$.
The {\it inhomogeneous polynomial kernel function}
$k_{\le d}:\; \mathbb{R}^n \times \mathbb{R}^n \longrightarrow \mathbb{R}$
is given by $k_{\le d}(x,y)= (\theta \cdot\langle x,y\rangle +1)^d$,
where $\langle ., .\rangle$ denotes the standard scalar
product.
\end{Def}

This definition differs slightly from the
usual one which is obtained after dividing by~$\theta^d$.
Since $\theta$ is chosen arbitrarily in $(0,1)$, no qualitative
change is introduced by our convention. It is however, as we will see,
the more natural one.

It is well-known that {\it feature space} $\calF_{\le d}$
of~$k_{\le d}$ satisfies $\calF_{\le d} \cong \R^m$, where
$m=\binom{n+d}{d}$, and that the {\it feature map}
$\Phi_{\le d}:\; \R^n \longrightarrow \calF_{\le d}$
is given by $\Phi_{\le d}(x)= (\gamma_\alpha x^\alpha \mid
|\alpha | \le d)$ (e.g., see~\cite{Scholkopf02}, Sec. 2.1).
Here we let $x^\alpha=c_1^{\alpha_1} \cdots c_n^{\alpha_n}$
and $|\alpha |=\alpha_1+\cdots+\alpha_n$
for $x=(c_1,\dots,c_n)$ and $\alpha=(\alpha_1,\dots,\alpha_n)$.
Moreover, we have $\gamma_\alpha=\sqrt{\theta^{|\alpha |}\cdot
\binom{d}{|\alpha |}\cdot \binom{|\alpha |}{\alpha_1, \dots,\alpha_n}}$.
To define $\Phi_{\le d}$ uniquely, we
order $p_1,\dots,p_m$ increasingly with respect to
the degree-lexicographic term ordering.
The feature map is characterized by the property that
$k_{\le d}(x,y)=\langle \Phi_{\le d}(x),\Phi_{\le d}(y)\rangle$
for $x,y\in\R^d$, i.e., by the fact that it transforms the
kernel function to a standard scalar product on the feature space.

On the algebraic side, the main objects linked via duality are vector
spaces of polynomials. Let $\vt =(t_1,\dots,t_n)$ be a tuple of
indeterminates, and let $\R[\vt]$ be the polynomial ring in these
indeterminates. For every $d\ge 0$, we denote by $\R[\vt]_{\le d}$
the vector space of all polynomials of degree at most~$d$.
Recall that the dimension of $\R[\vt]_{\le d}$ is $\binom{n+d}{d}$.
The inhomogeneous polynomial kernel function can be extended
to a pairing of polynomials as follows.

\begin{Not}
Let $\theta\in (0,1)$ and $d\ge 1$.
The inhomogeneous polynomial kernel on $\R[\vt]$ is the map
$\mathbf{k}_{\le d}:\ \R[\vt]^n \times \R[\vt]^n \longrightarrow
\R[\vt]$ given by $\mathbf{k}_{\le d}(F,G)= (\theta \cdot \langle F,G\rangle
+1)^d$ where we define $\langle F,G\rangle= f_1g_1+\cdots +f_ng_n$ for
$F=(f_1,\dots,f_n)$ and $G=(g_1,\dots,g_n)$.
\end{Not}

The polynomials $\mathbf{k}_{\le d}(x,\vt) \in\R[\vt]$
can be used to generate $\R(\vt]_{\le d}$, as our next proposition
shows. By the multinomial theorem, they can be expressed
as $\mathbf{k}_{\le d}(x,\vt) = (\theta \langle x,\vt\rangle +1)^d =
\sum_{|\alpha|\le d} \gamma_\alpha^2 x^\alpha t^\alpha$.
Recall that points $x_1,\dots,x_m\in\R^n$ are called {\it generic}
for a property~$\calP$ if there exists a Zariski open subset\footnote{
All such sets have full measure under any continuous probability density.
}
of
$U\subseteq (\R^n)^m$ such that property~$\calP$ holds for all
$(x_1,\dots,x_m)\in U$.

\begin{Prop}\label{Prop:dual}
Let $d\ge 0$, let $m'\ge \binom{n+d}{d}$, and
let $x_1,\dots, x_{m'}\in \R^n$ be generic. Then we have
$\R[\vt]_{\le d} = \langle \mathbf{k}_{\le d}(x_i,\vt) \mid 1\le i\le m'\rangle$.
\end{Prop}

The next proposition says that $\R[\vt]_{\le d}$ is dual to the feature
space $\calF_{\le d}$.

\begin{Prop}\label{Prop:isomphi}
(a) For $d\ge 1$, the map $\phi:\, \R[\vt]_{\le d} \longrightarrow
(\calF_{\le d})^\vee \cong {\rm Hom}_{\R}(\R^m,\R)$ defined by
$f=\sum_{|\alpha |\le d}\, c_\alpha {\vt}^\alpha \mapsto f^\vee$,
where $f^\vee(e_\alpha)=c_\alpha/\gamma_\alpha$, is an
isomorphism of $\R$-vector spaces.\\
(b) For $d\ge 1$, the map $\psi :\, \R[\vt]_{\le d} \longrightarrow
\calF_{\le d}$ defined by $f=\sum_{|\alpha |\le d}\,
c_\alpha {\vt}^\alpha \mapsto \sum_{|\alpha |\le d}
(c_\alpha/\gamma_\alpha) e_\alpha$ is an isomorphism of $\R$-vector
spaces.
\end{Prop}

By dualizing the map~$\phi$ and using the canonical isomorphism
between $\calF_{\le d}$ and its bidual, we obtain an isomorphism
$\phi^\vee:\, \calF_{\le d} \longrightarrow (\R[\vt]_{\le d})^\vee$
which maps~$e_\alpha$ to the $\R$-linear map given by
$t^\beta \mapsto \gamma_\alpha\,\delta_{\alpha\beta}$.
Another consequence of the preceding proposition is that,
if we pass to the union $\R[\vt]= \bigcup_{d\ge 1} \R[\vt]_{\le d}$,
we see that the polynomial ring contains the duals of all
feature spaces $\calF_{\le d}$.

Finally, we interpret the
isomorphism~$\phi$ in terms of kernel decision functions.
Recall that the map $\eval(f):\,
\R^n \longrightarrow \R$ given by $x\mapsto f(x)$
is called the polynomial function associated to
$f\in\R[\vt]$ and that the polynomial function
associated to $\mathbf{k}_{\le d}(x,\vt)$ is called a {\it kernel
decision function}. Let us denote the vector space of all kernel
decision functions by $\KDF_{\le d}$.

\begin{Cor}\label{Cor:KDF}
The map $\Phi_{\le d}^\ast:\, (\calF_{\le d})^\vee \longrightarrow
\KDF_{\le d}$ given by $\ell \mapsto \ell\circ\Phi_{\le d}$
is an isomorphism.
\end{Cor}

Now we use the isomorphism $\phi:\ \R[\vt]_{\le d}
\longrightarrow (\calF_{\le d})^\vee$ to
transfer the standard scalar product on $\calF_{\le d}$
to $\R[\vt]_{\le d}$ in the natural way. The result is
a scalar product $\langle ., .\rangle_\phi$
on~$\R[\vt]_{\le d}$ such that
$$
\langle \vt^\alpha, \vt^\beta\rangle_\phi \;=\;
\langle (1/\gamma_\alpha)\, e_\alpha^\vee,\,
(1/\gamma_\beta)\, e_\beta^\vee \rangle \;=\;
(1/\gamma_\alpha^2)\, \delta_{\alpha\beta}
$$
for $| \alpha | \le d$.
The next proposition provides a basic property
of the scalar product $\langle ., . \rangle_\phi$.

\begin{Prop}\label{Prop:scalarprod}
For $f\in\R[\vt]_{\le d}$ and $x\in\R^n$,
we have $f(x) = \langle f, \mathbf{k}_{\le d}(x,\vt) \rangle_\phi$.
\end{Prop}

The duality expressed by the map~$\phi$ is an algebraic analogue
of the theory of reproducing kernel Hilbert spaces.
The associated Hilbert space is the space of
polynomial functions $f:\R^n\rightarrow \R$
of degree $\le d$ which can be identified with $\R[\vt]_{\le d}$
by replacing the polynomial function~$f$ with the corresponding symbolic
polynomial. The equation in the preceding proposition could
also be obtained by combining the Riesz representation with this
identification (see~\cite{Scholkopf02}, Sec.~2.2). In the next section we
go beyond what can be shown using the usual RKHS duality alone.

\subsection{Duality of vanishing ideals and feature spans}
Next we show that ideals -- a classical concept in algebra -- are the proper
dual objects of feature spans of manifolds, in the same way as the polynomial ring
is the dual of feature space itself.
Recall that an {\it ideal}~$I$ in~$\R[\vt]$ is a vector subspace such that
$I\cdot \R[\vt] \subseteq I$.
Ideals are connected to subsets of~$\R^n$ as follows.

\begin{Def}
(a) Given a subset~$\calX$ of~$\R^n$, the set of polynomials
$\Id(\calX)=\{ f\in\R[\vt] \mid f(x)=0$ for all $x\in \calX\}$ is an ideal
in~$\R[\vt]$. It is called the {\it vanishing ideal} of~$I$.\\
(b) Given an ideal~$I$ in~$\R[\vt]$, the set of points
$\calV(I)=\{ x\in\R^n \mid f(x)=0$ for all $f\in I\}$ is called the
{\it zero set} of~$I$. A subset $\calX$ of~$\R^n$ is called an {\it algebraic set}
if it is the zero set of an ideal in~$\R[\vt]$.
\end{Def}

Notice that not every ideal in~$\R[\vt]$ is a vanishing ideal,
since vanishing ideals have the additional property of being
{\it radical}, i.e., if $f\in I^i$ for some $i\ge 1$ then
$f\in I$. Given an ideal~$I$ in~$\R[\vt]$ and $d\ge 0$,
we let $I_{\le d}=I \cap \R[\vt]_{\le d}$.

From now on the data manifold~$\calX$
is always assumed to be an algebraic set in~$\R^n$.
In this case, the vanishing ideal $\Id(\calX)$ is dual
to the manifold $\calX$ in the following sense.

\begin{Thm}\label{Thm:dual-manifold}
Let $\calX\subset \R^n$ be an algebraic set. \\
(a) We have $\Id(\calX)_{\le d} = \langle \mathbf{k}_{\le d}(x,\vt)
\mid x\in \calX\rangle^{\perp}$, where $\perp$ is taken with
respect to~$\langle ., .\rangle_\phi$.\\
(b) Under the isomorphism~$\phi$, the set $\Id(\calX)_{\le d}$ corresponds to
$\{ \ell \in (\calF_{\le d})^\vee \mid \ell(\langle \Phi_{\le d}(\calX)
\rangle) =0 \}$.\\
(c) Under the isomorphism~$\psi$, the set $\Id(\calX)_{\le d}$ corresponds to
the feature span $\langle \Phi_{\le d}(\calX)\rangle$.
\end{Thm}

The space $\langle \Phi_{\le d}(\calX)\rangle$ will be called the {\rm feature span}
of~$\calX$. As discussed above, Theorem~\ref{Thm:dual-manifold} is a kernelized
version of the usual algebra-geometry duality based on Hilbert's
Nullstellensatz.
\begin{Rem}
All the results in this section hold, mutatis mutandis, for homogeneous polynomial
kernels.  For other kernels, they can be adapted via kernelizing the polynomial
ring, but then exact statements may need to be replaced with approximate ones.
For this reason, and to save space, we work with inhomogeneous polynomial kernels from now on.
\end{Rem}

\section{Vanishing ideals and cross-kernel matrices}
\label{Sec:random}

In this section we discuss structural and algebraic properties of
cross-kernel matrices between data points and further sampled points.
As in the previous section we use inhomogeneous
polynomial kernels. Similar statements can be proven for homogeneous
polynomial kernels, and even non-polynomial kernels. In the latter case our
exact statements have to be transformed into spectral approximation results.

The following setting is used throughout the section.
There is a manifold $\calX\subseteq \R^n$ from which data points are sampled,
and a manifold $\calZ\subseteq \R^n$ from which further points are sampled
in such a way that they feature-span~$\calZ$.
We assume that both~$\calX$ and~$\calZ$ are algebraic sets and that
there is a degree $d\ge 1$ such that~$\calX$ and~$\calZ$
are cut out by polynomials of degree at most~$d$. Furthermore,
we assume $\calX\subseteq \calZ$. The data manifold~$\calX$ is considered to be
fixed and unknown, while~$\calZ$ can be chosen by the experimenter
and is typically given by $\calZ=\R^n$. The sampled data points
are denoted by $x_1,\dots,x_n \in\calX$, and the points samples from~$\calZ$
are denoted by $z_1,\dots,z_M\in\calZ$.
They are the rows of the matrices $X\in\R^{N\times n}$ and $Z\in\R^{M\times n}$,
respectively. We begin by introducing kernel and cross-kernel matrices.

\begin{Def}
(a) Given $d\ge 1$ and the inhomogeneous polynomial kernel
$k_{\le d}: \R^n \times \R^n \longrightarrow \R$, we denote the
matrix of size $N\times M$ whose entry in position $(i,j)$ is
$k_{\le d}(x_i,z_j)$ by $K_{\le d}(X,Z)$. This matrix is called the
{\it cross-kernel matrix} between~$X$ and~$Z$.\\
(b) The matrix $K_{\le d}(X,X)$ is simply called the
{\it kernel matrix} of~$X$.
\end{Def}

In the following we study properties of the matrices $K_{\le d}(X,X)$,
$K_{\le d}(X,Z)$ and $K_{\le d}(Z,Z)$, and we prove duality statements
between cross-kernel matrices and vanishing ideals.
An important condition is \emph{genericity} of the points sampled from~$\calZ$.
In our case, the appropriate genericity condition is that the points
feature-span~$\calZ$. It is defined as follows.

\begin{Def}
Let $d\ge 1$, let $\Phi_{\le d}:\, \R^n \longrightarrow \calF_{\le d}$
be the feature map, and let $z_1,\dots,z_M\in\calZ$.\\
(a) For a set $S\subseteq \R^n$, the $\R$-linear span
$\fspan(S) =\langle \Phi_{\le d}(s) \mid s\in S\rangle$
is called the {\it feature span} of~$\calS$, and the number
$\frk S = \dim_{\R}(\fspan(S))$ is called the {\it feature rank} of~$S$.\\
(b) We say that $z_1,\dots, z_M\in\calZ$ {\it feature-span} the algebraic
set~$\calZ$ if $\fspan(\{z_1,\dots,z_M\}) = \fspan(\calZ)$.
By a slight abuse of notation, we also say that the matrix~$Z$
feature-spans $\calZ$ in this case, and we write $\frk(Z)$
for $\frk(\{z_1,\dots,z_M\})$. Moreover, we say that $z_1,\dots, z_M$
are {\it feature independent} if
$\Phi_{\le d}(z_1),\dots,\Phi_{\le d}(z_M)$ are linearly independent
in~$\calF_{\le d}$.
\end{Def}

The matrix~$Z$ is also called the {\it feature generating matrix}.
Notice that, by elementary linear algebra, a set of points $\{z_1,\dots,z_M\}$
feature-spans $\calZ$ if and only if $\frk(Z) = \frk(\calZ)$.
A traditional method to analyse the kernel matrix $K_{\le d}(X,X)$
is to {\it subsample}~$X$, i.e., to choose the points $z_i$ in
$\{x_1,\dots,x_N\}$. This entails the common problem that one cannot
assume that~$Z$ feature-spans~$\calX$. Our next theorem
characterizes the matrices~$Z$ which avoid this problem.

\begin{Thm}\label{Thm:CKMprops}
In the above setting, assume that \ $\fspan(X) \subseteq\fspan(Z)$.\\
(a) {\rm (Cross-Kernel Rank)} \ We have $\rk K_{\le d}(X,Z) = \frk(X)$.
Hence, if~$X$ feature-spans~$\calX$ then $\rk K_{\le d}(X,Z) = \frk(\calX)$.\\
(b) {\rm (Cross-Kernel Data Manifold Certificate)} \ Let $c\in\R^n$ such
that $\Phi_{\le d}(c)\in\fspan(Z)$. Then we have $\Phi_{\le d}(c)\in\fspan(X)$
if and only if $(k_{\le d}(c,z_1),\dots,k_{\le d}(c,z_M))$ is contained in the
row span of~$K_{\le d}(X,Z)$. If~$X$ feature-spans $\calX$ then
this implies $c\in\calX$.\\
(c) {\rm (Cross-Kernel Nullspace)} \ Let $f\in \Id(\calZ)_{\le d}^\perp$
and suppose that~$X$ feature-spans~$\calX$ and~$Z$ feature-spans~$\calZ$.
Then we have $f\in \Id(\calX)_{\le d}$ if and only if~$f$ is of the form $f=\sum_{i=1}^M c_i \mathbf{k}_{\le d}(z_i,\vt)$ with a coefficient tuple $c=(c_1,\dots,c_M)\in \R^M$
such that $K_{\le d}(X,Z)\cdot c^\top =0$.\\
(d) {\rm (Cross-Kernel Range)} \ Let $f\in \Id(\calZ)_{\le d}^\perp$.
Suppose that~$X$ feature-spans~$\calX$, $Z$ feature-spans~$\calZ$,
and $z_1,\dots,z_M$ are feature-independent.
Then $f\in \Id(\calX)_{\le d}^\perp$ if and only if there is a
vector $(c_1,\dots,c_M)$ in the left range of $K_{\le d}(X,Z)
\cdot K_{\le d}(Z,Z)^{-1/2}$ such that
$f=\sum_{j=1}^M c_j \mathbf{k}_{\le d}(z_j,\vt)$.
\end{Thm}

The statements of this theorem can be turned into statements
about $K_{\le d}(X,X)$ and $K_{\le d}(Z,Z)$ by taking $\calX=\calZ$ and $X=Z$.
Let us denote the Moore-Penrose pseudoinverse of a matrix~$A$
by~$A^+$. The cross-kernel matrix allows us to reconstruct
the kernel matrix $K_{\le d}(X,X)$ as follows.

\begin{Thm}\label{Thm:matcompeq}
In the above setting we have
$K_{\le d}(X,X) =
K_{\le d}(X,Z)\cdot K_{\le d}(Z,Z)^+ \cdot K_{\le d}(Z,X)$
if and only if \ $\fspan(X)\subseteq \fspan(Z)$.
In particular, equality holds if~$Z$ feature-spans~$\calZ$.
\end{Thm}

The equality in Theorem~\ref{Thm:matcompeq} is an exact
matrix decomposition of the kernel matrix of~$X$ which does not require~$Z$
to be subsampled from~$X$. This differs markedly from Nystr\"om type
methods.

\section{The ideal-kernel duality and kernel learning}

We demonstrate how kernel-ideal duality and cross-kernel matrices can be employed for common learning tasks. They enable us to improve
computational cost, stability, and to obtain novel information about
the data manifold. In the setting of the
preceding section, we collect some basic observations on the
matrix $K:=K_{\le d}(X,Z) \cdot K_{\le d}(Z,Z)^{-1/2}$
under the assumption that~$Z$ feature-spans $\calZ$ (or $\R^n$).

\begin{Rem} (a) Assuming that arithmetic and kernel function evaluation can be
done at cost $O(1)$, it takes $O(M^3+MN)$ steps to compute~$K$ and
$O(M^2N+M^3)$ steps to obtain a singular value decomposition of $K$.
Thus both tasks require asymptotically linear time in~$N$.\\
(b) By Theorem~\ref{Thm:matcompeq}, the equality
$K_{\le d}(X,X)=K\cdot K^\top$ holds {\it exactly}.
Thus a SVD for $K_{\le d}(X,X)$ can be obtained
from a SVD $K=USV^\top$ of~$K$ via $K_{\le d}(X,X) = US^2U^\top$.
By~(a), the cost of this method is only $O(M^2N+M^3)$, which is linear in~$N$,
while the direct algorithm costs $O(N^3)$.\\
(c) If $Z$ is chosen to be a subsample of~$X$, one recovers common subsampling
strategies. However, the feature generating matrix~$Z$ can be chosen completely
independently from~$X$. Hence the matrix $K_{\le d}(Z,Z)^{-1/2}$ needs to be
precomputed only once for a given kernel degree and data dimension.
In particular, it is not necessary to subsample~$X$. For instance, we can choose~$Z$
randomly. Or, if desired, the matrix~$Z$ can be chosen to contain both a subsample
of~$X$ and random points.\\
(d) As mentioned in Theorem~\ref{Thm:CKMprops}.b, the right singular vectors
of~$K$ characterize the data manifold~$\calX$ from which~$X$ is sampled.
These right singular vectors cannot be obtained from $K_{\le d}(X,X)$ alone.
\end{Rem}

\subsection*{The Ideal PCA (IPCA) algorithm}
The kernel-ideal duality can be used to both speed up kernel PCA and additionally
yield feature functions cutting out the data manifold. By using the
matrix $K_{\le d}(X,Z)$ instead of $K_{\le d}(X,X)$, we obtain
the same feature projections and principal components as in ordinary
kernel PCA  (both in terms of the $X$- and the $Z$-basis), the
corresponding singular values, and novel feature projections which cut out
the data manifold (which we call ``right principal components'').
All of this is achieved by Algorithm~\ref{Alg:features} in linear time
and without the necessity of subsampling.

\begin{algorithm}[ht]
\caption[\texttt{IPCA} Computes left and right principal vectors of the cross-kernel]{\texttt{IPCA} Computes left and right principal vectors of the cross-kernel.\\
\textit{Input:} a degree $d\ge 1$, a data matrix~$X\in\R^{N\times n}$,
a matrix~$Z\in\R^{M\times n}$. A threshold $\epsilon>0$ or a cut-off $m$.\\
\textit{Output:} $m$ left principal components as columns of a matrix
$U\in\R^{N\times m}$ and $m$ right principal components as columns of
a matrix $V\in\R^{M\times m}$, as well as the corresponding singular values
on the diagonal of a matrix $S\in\R^{m\times m}$.
\label{Alg:features}}
\begin{algorithmic}[1]
    \STATE Compute the matrices $\Kxz=K_{\le d}(X,Z)\in\R^{N\times M}$
           and $\Kzz=K_{\le d}(Z,Z)\in\R^{M\times M}$.
    \STATE Compute $K=\Kxz\cdot \Kzz^{-1/2}$.
    \STATE If centering is desired, replace $K$ by
           $K-\frac{1}{N}\mathbf{1} K,$ where $\mathbf{1}\in\R^{N\times N}$ denotes
           the all-ones matrix.
	\STATE Compute the truncated singular value decomposition
	       $U S V^\top$ of~$K$, with $U\in\R^{N\times m}, V\in \R^{M\times m}$,
	       and $S=\diag (\sigma_1,\dots,\sigma_m)$. (If $\epsilon$ is given, then $\sigma_m$ is the
	       smallest singular value with $\sigma_m\ge \epsilon$.)
    \STATE Output $U,V,S$.
\end{algorithmic}
\end{algorithm}

For dimension reduction tasks, the important quantities are evaluations of
the principal components. The left principal components agree with the
ones from PCA, while the orthogonal of the space generated by the
right principal components is used in Algorithm~\ref{Alg:evalfeatures}
to project onto the data manifold.

\begin{algorithm}[ht]
\caption[\texttt{eval-IPCA} Computes IPCA features of the data]
{\texttt{eval-IPCA} Computes IPCA features of the data.\\
\textit{Input:} a degree $d\ge 1$, a data matrix $X\in\R^{N\times n}$,
a matrix~$Z\in R^{M\times n}$, a threshold $\epsilon>0$,
and a test data point $x_t\in\R^n$.\\
\textit{Output:} left principal features $u\in \R^m$ (PCA-like), right principal
features $v\in\R^m$ (PCA-like) and certifying features $v^\perp \in \R^m$ (IPCA-like). The entries of~$u$ and $v$ are local, whitened coordinates on
the data manifold~$\calX$ ($u$ in $X$- and $v$ in $Z$-basis), and the entries of~$v^\perp$ are orthogonal coordinates
on~$\calX$ (i.e., they should be close to zero for $x_t\in\calX$).
 \label{Alg:evalfeatures}}
\begin{algorithmic}[1]
    \STATE Apply the \texttt{IPCA} algorithm to obtain matrices
           $U,S,V$. The columns of $U\in\R^{N\times m}$ contain $m$ left
           principal vectors and the columns of $V\in\R^{M\times m}$ contain
           $m$ right principal vectors.
    \STATE Compute the evaluation vectors $\kappa_X=K(x_t,X)$ and $\kappa_Z=K(x_t,Z)$.
    \STATE Compute (or pass over from IPCA) the matrix $\Kzz^{-1/2}=K_{\le d}(Z,Z)^{-1/2}$.
	\STATE Compute $P = I - VV^\top$, where $I$ is the identity matrix
	       of size~$M$.
    \STATE Output: $u = \kappa_X\cdot U\cdot S^{-1}\quad v = \kappa_Z\cdot V\cdot S^{-1} \quad v^\perp = \kappa_Z\cdot \Kzz^{-1/2}\cdot P$
\end{algorithmic}
\end{algorithm}

The vector $u$ contains exactly the principal feature evaluations of~$x$
one would obtain from kernel PCA. The vector $v$ is the IPCA-type right analogue, while~$v^\perp$ can be seen as a measure for
how closely~$x$ lies on the data manifold. Neither $v$ nor $v^\perp$ can be obtained from any algorithm
involving only the matrix $K_{\le d}(X,X)$.

\section{Experimental Validation}
\label{Sec:KPCA}

We validate the main claims of our paper experimentally. We show, in this section: (a) the left principal IPCA features obtained from the cross-kernel are equal to the principal features obtained from kernel PCA~\cite{KPCA1998} and the kernel matrix $\Kxx$ (\ref{Sec:kPCA}) but faster to obtain, (b) the vanishing features obtained from IPCA can be used for manifold learning with kernels, and (c) the IPCA vanishing features outperform kernel PCA features in the USPS classification scenario in terms of compactness, computational cost, and accuracy.

{\bf Data sets.}
Our synthetic data below are: (i) points uniformly from a circle of radius $10$, then
perturbed by independent $2$-dimensional standard, Gaussian random vectors, and (ii) $2$ circles
on a sphere of radius $5$. For each circle, $200$ points are sampled uniformly, then $3$-dimensional Gaussian noise with variance $0.1$ is added. The real world data set used below is (iii) the USPS handwritten digits data set~\cite{USPS}. All experiments below are done with the inhomogenous polynomial kernel of degree $2$, with $\theta = 1$.

{\bf Kernel PCA vs IPCA.}
\label{Sec:kPCA}
We compared kernel PCA features and IPCA left features on the (ii) $2$ circles data. We fixed $M=12$ as the number of feature spanning points, and computed principal components for $N=10,20,\dots, 1000$ data points. The six non-degenerate principal vectors and their principal values were equal to the IPCA left feature vectors and the squares of the corresponding singular values, up to machine precision. Figure~\ref{fig:PCA_runtime} shows log-runtimes in comparison. The runtime for IPCA is lower by orders of magnitude.
\begin{figure}[htbp]
	\centering
    \includegraphics[width=0.45\textwidth]{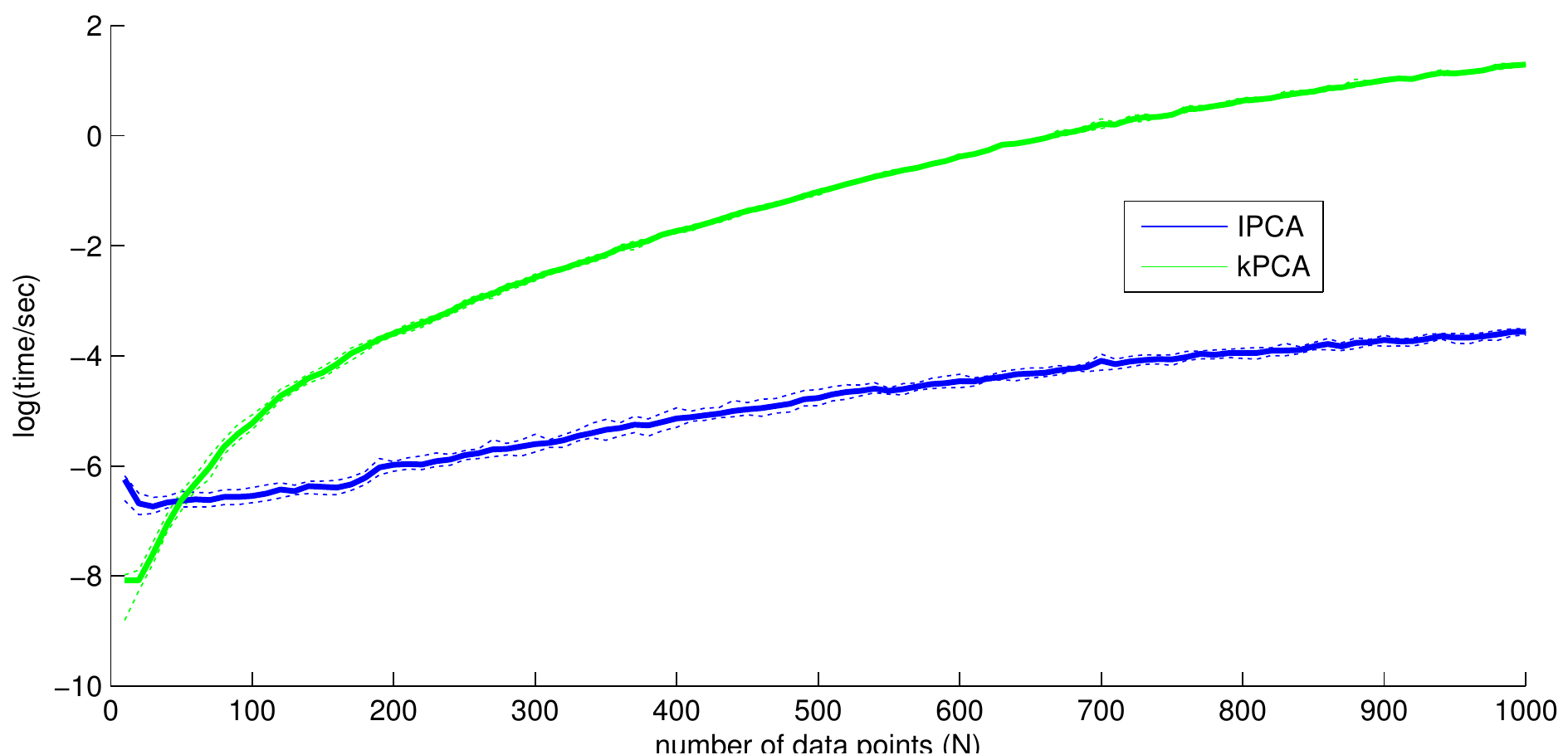}
	\caption{Runtime of kernel PCA and IPCA. Thick line is mean, dotted lines are 0.1 and 0.9 quantiles over $20$ repetitions.}
	\label{fig:PCA_runtime}
\end{figure}

{\bf Manifold learning with cross-kernel and IPCA.}
\label{Sec:dimred}
We used IPCA to learn the data manifold from 200 samples in the (i) circle and the (ii) 2 circle example. The manifold is the approximated by the points with low evaluation of the normalized IPCA certifying features, the result can be seen in figure \ref{fig:manifoldEst}. The singular value threshold was fixed to $m=5$ for (i) and $m=9$ for (ii). The certifying features were evaluated on a fine grid, grid points with evaluation norm near the training data were considered in-manifold, with 0.1-quantile for (i) and 0.01-quantile for (ii).

\begin{figure}[htbp]
	\centering
		\subfigure[]{\includegraphics[width=0.35\textwidth]{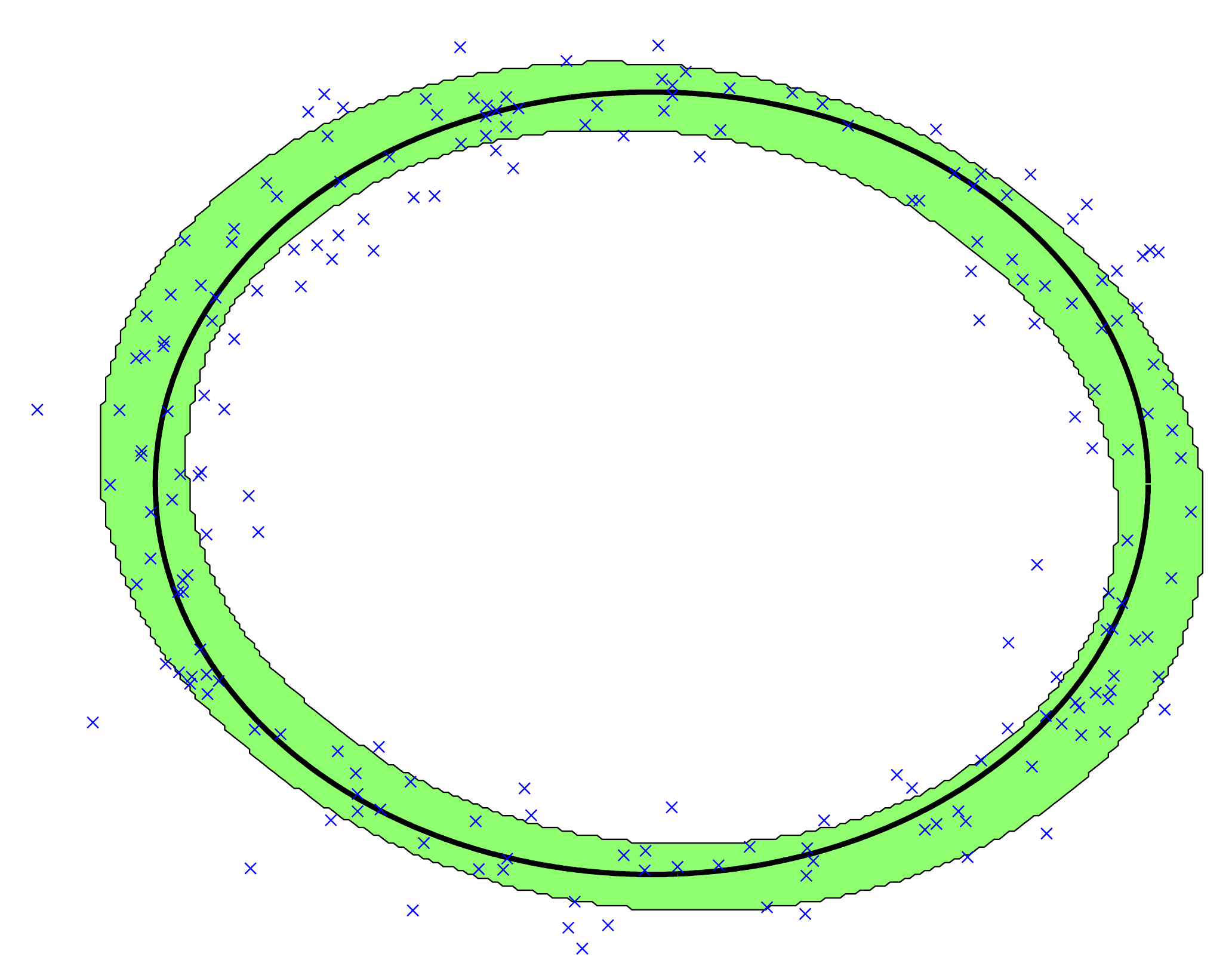}\label{fig:manifoldEst.circle}}\hfill
		\subfigure[]{\includegraphics[width=0.4\textwidth]{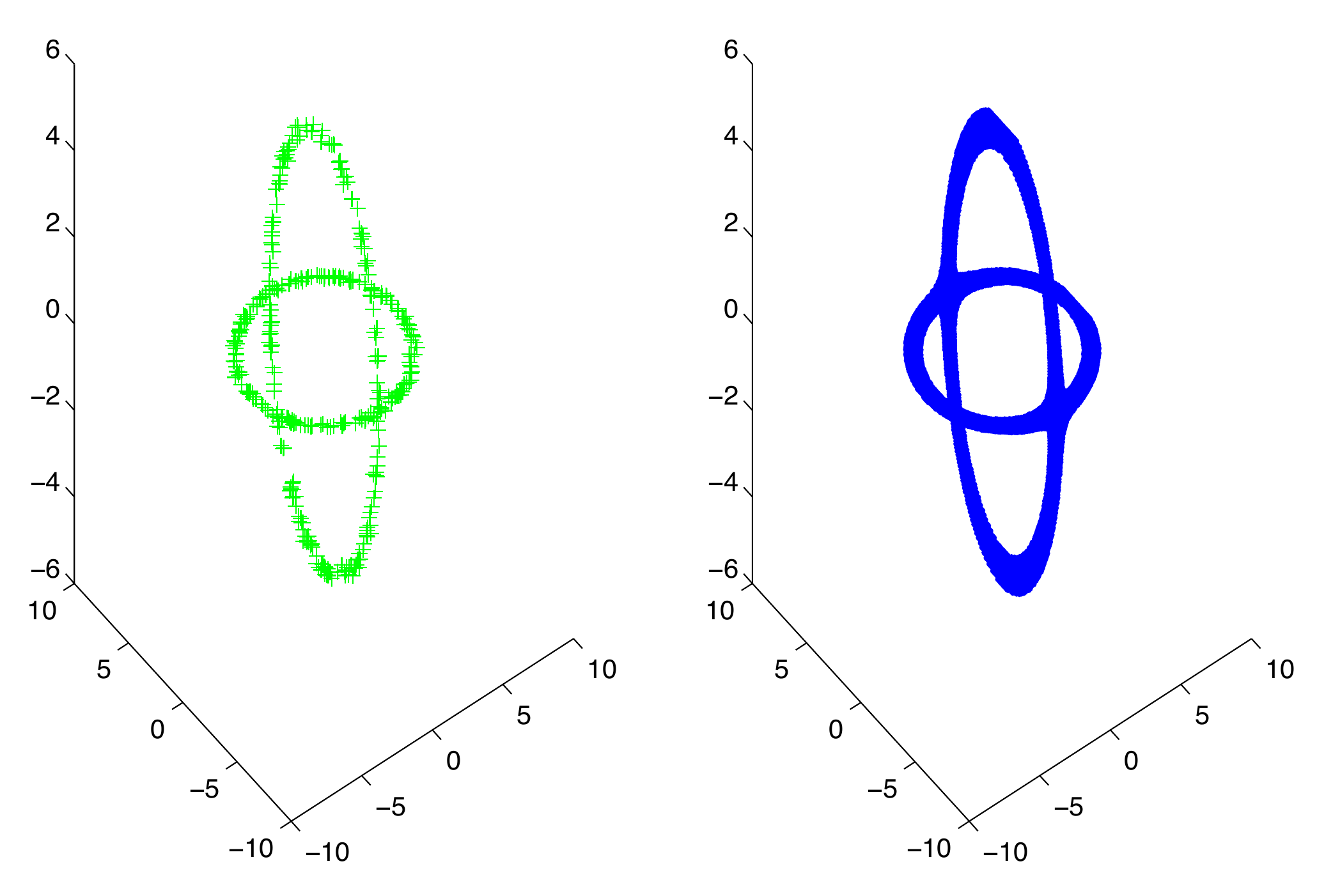}\label{fig:manifoldEst.sphere}}
	\caption{Manifold estimation with IPCA. The crosses are the data points. The estimated manifolds are the light green ring in~\ref{fig:manifoldEst.circle} and the two blue rings in~\ref{fig:manifoldEst.sphere}.}
	\label{fig:manifoldEst}
\end{figure}

{\bf Classification with IPCA certifying vectors.}
\label{Sec:class}
We used the IPCA features for classification on the USPS data set. In all experiments,
$700$ training examples are used, and the rest of the USPS data set is used for evaluation.

Denote the different training classes $X_1,\ldots, X_t$ and the pooled data by $X_{\text{train}} := \cup_{i=1}^t X_i$.  The natural form of a one-vs-all classifier using IPCA certifying features is as follows: (1) to train, compute $\operatorname{IPCA}(X_i,Z)$
for each of the $X_i$ with the same fixed $Z$; to classify a set of points $X_{\text{test}}$,
compute the right-evaluation vectors using Algorithm \ref{Alg:evalfeatures} for each class
and assign the label corresponding to the minimum norm evaluation.  The
correctness of this construction follows directly from Theorem \ref{Thm:CKMprops}. All
the stability and convergence guarantees of IPCA go through without modification.

We compare three choices for the feature generating points $Z$: (1) choose the feature-generating matrix $Z$
with standard normal Gaussian entries; (2) choose it as a uniform random sub-sample of $X_{\text{train}}$; (3) choose it as a degenerate sub-sample of $X_{\text{train}}$; we simulate a high number of repetitions by sub-sampling from a sub-sample of size $N/4$.

Figure \ref{fig:classification}(a) summarizes the results. Non-degenerate sub-sampling (2) is best, since it takes advantage of the implicit dimension-reduction of expressing the $\calX_i$ in a coordinate system derived
$\calX\subset \mathbb{R}^n$, which can be of much lower dimension than $n$.  However, when the
sampling is \emph{degenerate} as in (3), e.g., with many repetitions, a random feature-generating matrix of the
same size as in (1) gives better performance. Eventually, with growing $M$ all choices of $Z$ yield comparable results. As it can not be distinguished a-priori whether the training sample is degenerate or not, random $Z$ can be preferrable.

We also compared the IPCA certifying features to the IPCA left and right principal features (PCA-like), by extracting $m=32$ features, then performing linear classification via LIBLINEAR \cite{LIBLINEAR}. The feature-generating matrix $Z$ is chosen randomly with standard normal entries. Figure \ref{fig:classification}(b)--(c) show that $m=32$ certifying features yields more accurate and faster results than the same number of principal features. The results reported in~\cite{KPCA1998} for classical kernel PCA are comparable to the PCA-like features, which all are outperformed by the IPCA certifying features.

\begin{figure}[htbp]
	\centering
		\subfigure[]{\includegraphics[width=0.3\textwidth]{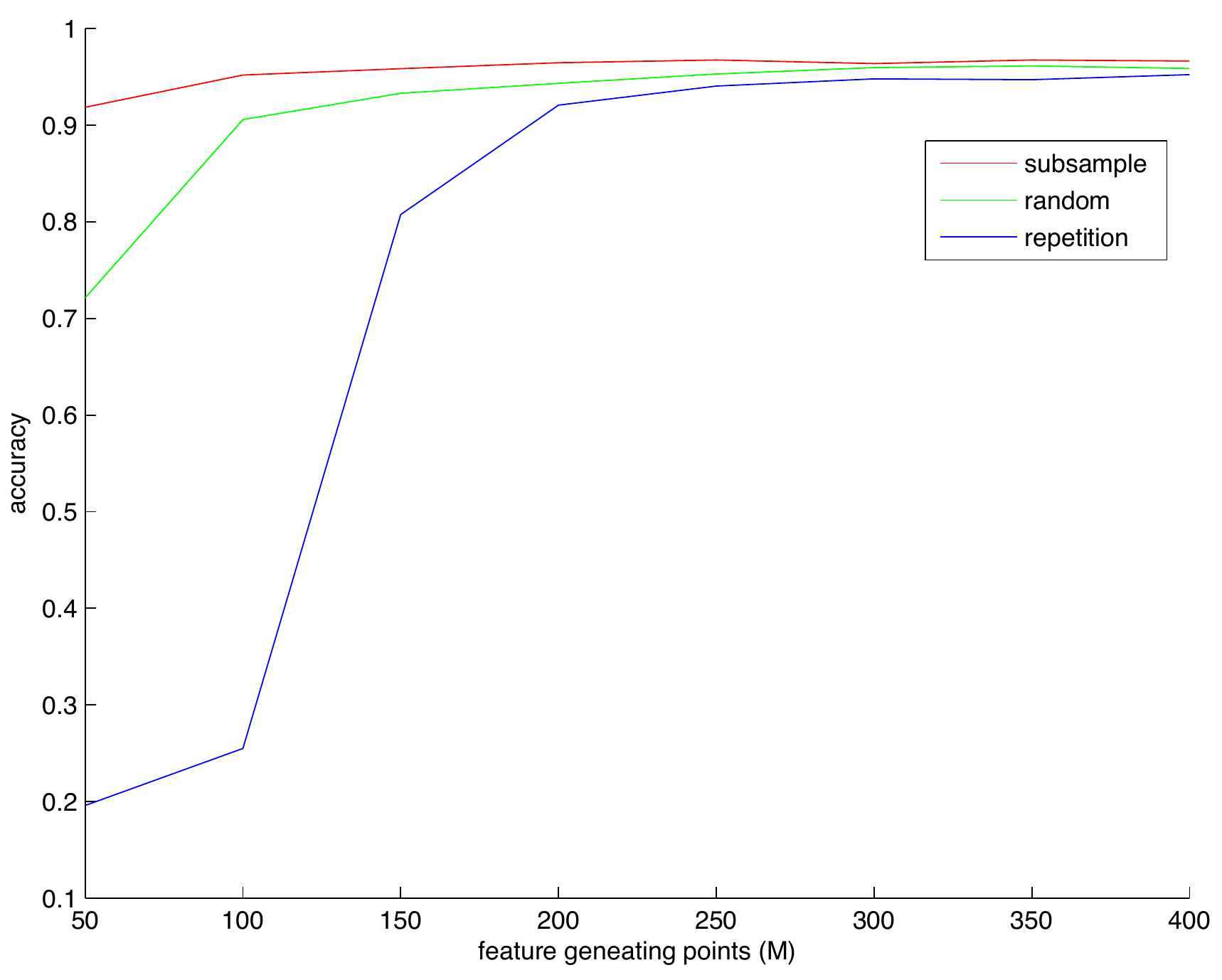}}
		\subfigure[]{\includegraphics[width=0.3\textwidth]{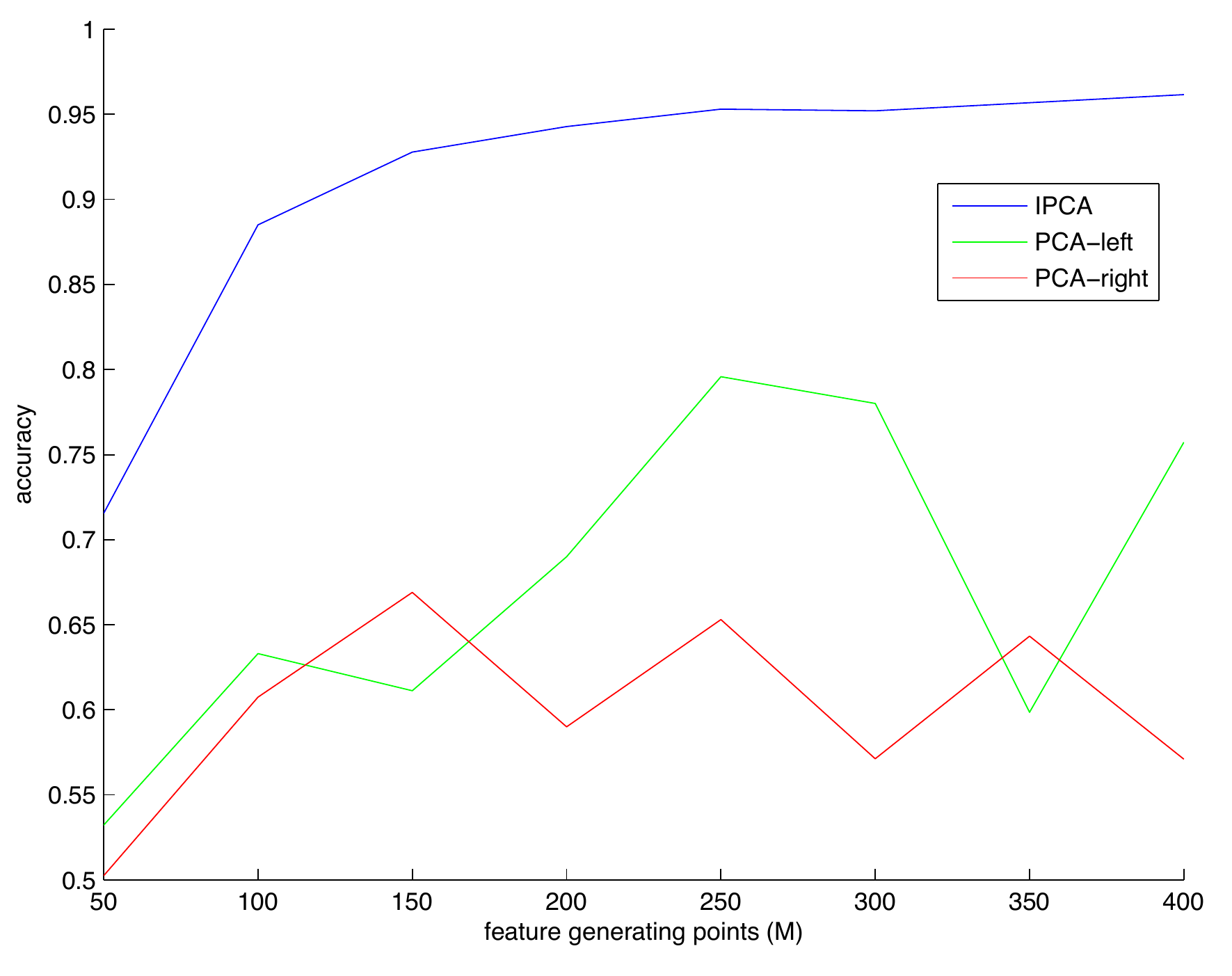}}
		\subfigure[]{\includegraphics[width=0.3\textwidth]{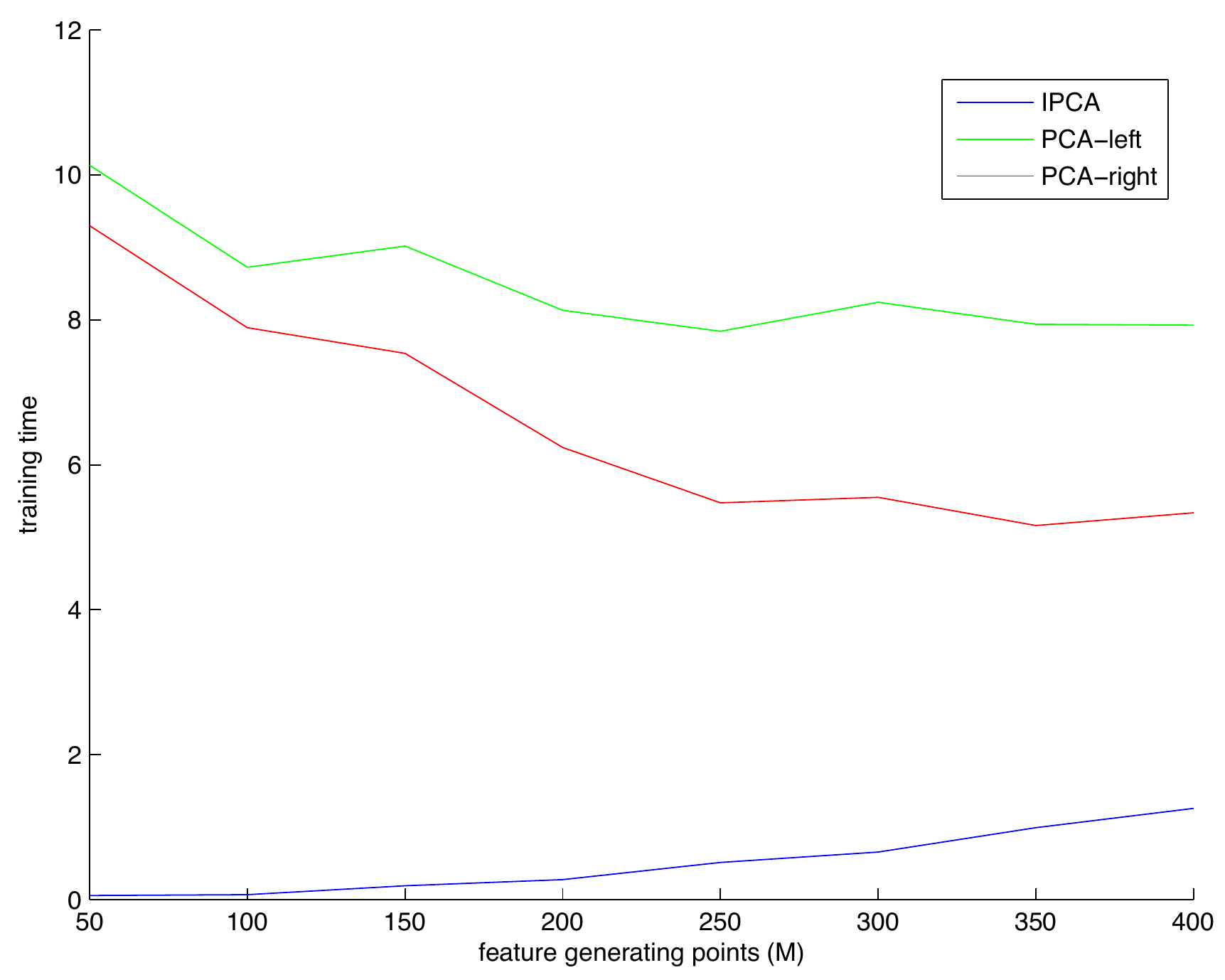}}
	\caption{(a) Classification with IPCA certifying features, comparing random $Z$ and subsampled $Z$ from generic and degenerate sampling; (b) performance comparison of the three IPCA features - left and right principal features are closely related to kernel PCA features; (c) speed comparison of IPCA feature based classifiers.}
	\label{fig:classification}
\end{figure}

Above comparing different modes of instanciating IPCA, the experiments also show that IPCA is capable of extracting competitive features for considerably less and better scaling computational cost than kernel PCA, and achieves comparable performance on a much smaller number of relevant features.

%
%

\subsubsection*{Acknowledgments}
%
LT is supported by the European Research Council under the European Union's Seventh Framework
Programme (FP7/2007-2013) / ERC grant agreement no 247029-SDModels.  This research was conducted
while MK and LT were guests of FK at Mathematisches Forschungsinstitut Oberwolfach, supported
by FK's Oberwolfach Leibniz Fellowship.

\subsubsection*{References}

\small{

\begingroup
\let\oldsection\section
\renewcommand{\section}[2]{}%

\endgroup

} 

\newpage

\section*{Appendix}

{\normalsize

\begin{proof}{\it (of Proposition~\ref{Prop:dual})}\\
Let $m=\binom{n+d}{d}$.
Under the isomorphism $\R[\vt]_{\le d} \rightarrow
\R^m$ given by $\sum_{\alpha} c_\alpha \vt^\alpha
\mapsto (c_\alpha)_{|\alpha |\le d}$, the polynomial $\mathbf{k}_{\le d}(x,\vt)$
maps to $(\gamma_\alpha^2 x^\alpha)_{|\alpha |\le d}$
for every $x\in\R^n$. Up to scaling the coordinates by 
$\gamma_\alpha^2$, the image of this isomorphism is a Veronese variety.
Now the claim follows from the fact that this Veronese
variety is irreducible and non-degenerate.
\end{proof}

\begin{proof}{\it (of Proposition~\ref{Prop:isomphi})}\\
First we show~(a). Both vector spaces have the same dimension. 
The map~$\phi$ is $\R$-linear and sends the basis vector ${\vt}^\alpha$
to the basis vector $(1/\gamma_\alpha)\cdot e_\alpha^\ast$,
i.e., to the map defined by $e_\beta \mapsto (1/\gamma_\alpha)\,
\delta_{\alpha\beta}$. From this the claim follows.

To prove~(b), we use the standard scalar product on~$\calF_{\le d}$
to identify it with $(\calF_{\le d})^\vee$. This (non-canonical) 
isomorphism identifies the standard basis vector~$e_\alpha$
with $e_\alpha^\ast$. Thus the claim follows from~(a).
\end{proof}

\begin{proof}{\it (of Corollary~\ref{Cor:KDF})}\\
It is easy to check that $\Phi_{\le d}^\ast = 
\eval\circ\phi^{-1}$. Since both maps $\eval$ and $\phi^{-1}$ are
isomorphisms, the claim follows.
\end{proof}

\begin{proof}{\it (of Proposition~\ref{Prop:scalarprod})}\\
This follows by writing $f=\sum_{\alpha} c_\alpha \vt^\alpha$
with $c_\alpha\in \R$ and computing
$$
\langle f, \mathbf{k}_{\le d}(x,\vt)\rangle_\phi \;=\;
\langle {\textstyle\sum_\alpha} c_\alpha \vt^\alpha,\, 
{\textstyle\sum_\beta} \gamma_\beta^2 x^\beta \vt^\beta \rangle_\phi 
\;=\; {\textstyle\sum_\alpha} c_\alpha \gamma_\alpha^2 x^\alpha
\langle \vt^\alpha, \vt^\alpha \rangle_\phi 
\;=\; {\textstyle\sum_\alpha} c_\alpha x^\alpha  \;=\; f(x)
$$
\end{proof}

\begin{proof}{\it (of Theorem~\ref{Thm:dual-manifold})}\\
Part~(a) follows from the definition of~$\Id(\calX)$ and 
Prop.~\ref{Prop:scalarprod}. 

To prove~(b), we first
use~(a) in order to map $\Id(\calX)_{\le d}$ isomorphically 
into~$(\calF_{\le d})^\vee$. The result is the orthogonal
complement of $U=\langle \sum_{|\alpha |\le d} \gamma_\alpha x^\alpha
e_\alpha^\vee \mid x\in \calX \rangle$. Since we are using the standard scalar
products on~$\calF_{\le d}$ and its dual, this orthogonal complement
equals the annihilator of the preimage of~$U$ in~$\calF_{\le d}$.
That preimage is generated by the elements $\sum_{|\alpha |\le d}
\gamma_\alpha x^\alpha e_\alpha  =\Phi_{\le d}(x)$ with $x\in \calX$.
In other words, the vector space $U^\perp$ is the annihilator of
the feature span of~$\calX$, as claimed.

Finally, we note that~(c) follows from~(b) by identifying 
$\calF_{\le d}$ with its dual via the standard scalar product.
\end{proof}

\begin{proof}{\it (of Theorem~\ref{Thm:CKMprops})}\\
First we prove claim~(a). Letting $v_i=\Phi_{\le d}(x_i)\in \R^m$
and  $w_j=\Phi_{\le d}(z_j)\in \R^m$ for $i=1,\dots,N$ and
$j=1,\dots,M$, we have to show that
the rank of the matrix $G=(\langle v_i,w_j\rangle)_{i,j}$
equals the dimension of $\langle v_1,\dots,v_N\rangle$.
By removing columns and rows of~$G$, we may assume that $\{v_1,\dots,v_N\}$
and $\{w_1,\dots,w_M\}$ are linearly independent. Using the hypothesis
$\langle v_1,\dots,v_N\rangle \subseteq \langle w_1,\dots,w_M\rangle$,
we see that $M\ge N$ and may assume $v_i=w_i$ for $i=1,\dots,N$.
Then the first~$N$ columns of~$G$ are the Gram matrix of~$V=(v_1,\dots,v_N)$
and the claim follows from the fact that the rank of the Gram matrix
of~$V$ equals the rank of~$V$.

Next we show~(b). The implication ``$\Leftarrow$'' is trivially true.
To prove the converse, we note that the hypothesis implies 
that there are $a_1,\dots,a_N\in\R$ such that $\Phi_{\le d}(c) 
- \sum_{i=1}^N a_i \Phi_{\le d}(x_i)$ is orthogonal to $\fspan(Z)$.
Since this vector is contained in $\fspan(Z)$, it is zero.
To prove the additional claim, we note that $\Phi_{\le d}(c)
\in\fspan(X)$ implies that all linear forms in the annihilator
of~$\fspan(X)$ vanish at~$c$. Via $\phi^{-1}$, these linear forms
correspond to the polynomials in~$\Id(\calX)_{\le d}$.
Since~$\calX$ is cut out by these polynomials, we get $c\in\calX$,
as claimed.

For the proof of part~(c) we use the isomorphism~$\phi$ to write the
hypothesis $f\in\Id(\calZ)_{\le d}^\perp$ as $f=\sum_{j=1}^M
c_j \mathbf{k}_{\le d}(z_j,\vt)$ for some $c_j\in\R$. We have
$f\in\Id(\calX)$ if and only if $\phi(f)$ is orthogonal to
$\langle \phi(\mathbf{k}_{\le d}(x_i,\vt))\rangle$.
Therefore we require that we have
\begin{align*} 
0 &\;=\; \langle \phi(f), \phi(\mathbf{k}_{\le d}(x_i,\vt)) \rangle
\;=\; \langle {\textstyle\sum_{j=1}^M} c_j 
{\textstyle\sum_\alpha} \gamma_\alpha z_j^\alpha e_\alpha^\vee,\;
{\textstyle\sum_\alpha} \gamma_\alpha x_i^\alpha e_\alpha^\vee \rangle\\
&\;=\; {\textstyle\sum_\alpha} \gamma_\alpha^2 
\left( {\textstyle\sum_{j=1}^M} c_j z_j^\alpha \right) \, x_i^\alpha
\end{align*}
for $i=1,\dots,N$. On the other hand, the entry in position $(i,j)$
of~$K_{\le d}(X,Z)$ is $k_{\le d}(x_i,z_j) = \langle \Phi_{\le d}(x_i),\,
\Phi_{\le d}(z_j) \rangle = \langle \sum_\alpha \gamma_\alpha x_i^\alpha
e_\alpha,\, \sum_\alpha \gamma_\alpha z_j^\alpha e_\alpha \rangle =
\sum_\alpha \gamma_\alpha^2 x_i^\alpha z_j^\alpha$. Hence 
$(c_1,\dots,c_M)^\top$ is in the nullspace of $K_{\le d}(X,Z)$ if and only
if $\sum_{j=1}^M \sum_\alpha \gamma_\alpha^2 x_i^\alpha z_j^\alpha c_j =0$
for $i=1,\dots,N$. As we have seen, this is equivalent to $f\in\Id(\calX)$.

Finally we show~(d). Using Theorem~\ref{Thm:dual-manifold}, we see
that the hypothesis yields $\Id(\calX)^\perp = \langle
\mathbf{k}_{\le d}(x_i,\vt) \mid i=1,\dots,N\rangle$ and
$\Id(\calZ)^\perp = \langle\mathbf{k}_{\le d}(z_j,\vt) \mid 
j=1,\dots,M\rangle$. Therefore there exist numbers
$c_j\in \R$ such that we have $f=\sum_{j=1}^M
c_j \mathbf{k}_{\le d}(z_j,\vt)$, and the question
is whether there exists a representation
$f=\sum_{i=1}^N a_i \mathbf{k}_{\le d}(x_i,\vt)$ with $a_i\in\R$.
By the hypothesis, the matrix $K_{\le d}(Z,Z)$ is invertible, since it is the
Gram matrix of the full rank matrix $(\Phi_{\le d}(z_j))$.
So, knowing $(c_1,\dots,c_M)$ is equivalent to knowing
$$
(c_1,\dots,c_M)\cdot K_{\le d}(Z,Z) \;=\; 
\bigl( {\textstyle\sum_\alpha} \gamma_\alpha^2
\left( {\textstyle\sum_{\ell=1}^M} c_\ell z_\ell^\alpha \right) \, 
z_j^\alpha \bigr)_j
$$

To prove the implication ``$\Rightarrow$'' we can write the latter tuple as
$$
\bigl( {\textstyle\sum_\alpha} \gamma_\alpha^2 \left( {\textstyle\sum_{i=1}^N}
a_i x_i^\alpha \right) \, z_j^\alpha \bigr)_j \;=\; 
(a_1,\dots,a_n) \cdot K_{\le d}(X,Z)
$$
and conclude that $c = (a_1,\dots,a_n)\cdot K_{\le d}(X,Z)\cdot
K_{\le d}(Z,Z)^{-1}$. For the implication ``$\Leftarrow$'', we are given
the equalities
$$
{\textstyle\sum_\alpha} \gamma_\alpha^2 \left( {\textstyle\sum_{\ell=1}^M} 
c_\ell z_\ell^\alpha \right) \, z_j^\alpha \;=\;
{\textstyle\sum_\alpha} \gamma_\alpha^2 \left( {\textstyle\sum_{i=1}^N}
a_i x_i^\alpha \right) \, z_j^\alpha
$$
for $j=1,\dots,M$. This means that, for $j=1,\dots,M$, the vectors 
$\sum_{\ell=1}^M \gamma_\alpha c_\ell z_\ell^\alpha e_\alpha - 
\sum_{i=1}^N \gamma_\alpha a_i x_i^\alpha e_\alpha$ and 
$\sum_{j=1}^M \gamma_\alpha z_j^\alpha e_\alpha$ are orthogonal 
with respect to the standard scalar product in $\calF_{\le d}$. 
Thus the first vector is both orthogonal to $\fspan(Z)$ and contained 
in $\fspan(Z)$, i.e., it is zero.
Consequently, we get $\sum_{\ell=1}^M c_\ell z_\ell^\alpha =
\sum_{i=1}^N a_i x_i^\alpha$, and as we have seen, this yields
$f\in \Id(\calX)_{\le d}^\perp$, as claimed.
\end{proof}

\begin{proof}{\it (of Theorem~\ref{Thm:matcompeq})}\\
Let~$Y\in\R^{(N+M)\times n}$ be the matrix obtained by row concatenating~$X$ 
and~$Z$. Observe that, by construction, there are projection matrices~$P$ and~$Q$ 
such that $K(X,X)=P\cdot K(Y,Y)\cdot  P^\top$ and $K(X,Z)=P\cdot K(Y,Y)\cdot Q^\top$,
as well as $K(Z,Z)=Q\cdot K(Y,Y)\cdot Q^\top$.

To prove the implication ``$\Rightarrow$'', we suppose that
$\fspan(X)\nsubseteq \fspan(Z)$. By the rank-nullity theorem, this implies 
$\frk(Y) > \frk(Z)$. Thus Theorem~\ref{Thm:CKMprops}.a yields 
$\rk K_{\le d}(Z,Z) < \rk K_{\le d}(Y,Y)$. The claim then follows by
applying Lemma~\ref{Lem:matcomp} to $A=K_{\le d}(Y,Y)$ and $B=Q$ and $C= Q^\top$.

For the reverse implication, we first note that
$\fspan(X)\subseteq \fspan(Z)$ implies $\fspan(Y) = \fspan(Z)$. Thus 
we have $\frk(Y) = \frk(Z)$ and Theorem~\ref{Thm:CKMprops}.a shows 
$\rk K_{\le d}(Z,Z) = \rk K_{\le d}(Y,Y)$. Now the claim follows
by applying Lemma~\ref{Lem:matcomp} in the same way as before.
\end{proof}

The following Lemma provides the key ingredient for the
preceding proof.

\begin{Lem}\label{Lem:matcomp}
Let $A\in\R^{m\times n}$, let $B\in\R^{n\times k}$, and let $C\in\R^{\ell\times m}$. 
Then the matrix equality $A = AB (CAB)^+ CA$ holds if and only if \ $\rk A = \rk CAB$.
\end{Lem}

\begin{proof}
First we show the implication ``$\Leftarrow$''.
Since $\rk CAB = \rk A$, we have $\rk CA = \rk AB = \rk A$, and therefore 
$\colspan(CA) = \colspan(A)$ as well as $\rowspan(AB) = \rowspan(A)$. 
This implies $A=C^+CA=ABB^+$. The definition of the pseudoinverse
yields $CAB = CAB (CAB)^+ CAB$.
Multiplying this equality by~$C^+$ and $B^+$, we get
$A = C^+ CABB^+ = C^+ CAB (CAB)^+ CABB^+ = AB (CAB)^+ CA$,
as claimed.

To prove the reverse implication, suppose that $\rk A \neq \rk CAB$. 
Since the rank of a matrix does not increase by matrix multiplication,
we must have $\rk A> \rk CAB$. Since the pseudoinverse of~$CAB$
satisfies $\rk (CAB)^+ = \rk CAB$, we find 
$\rk A > \rk CAB = \rk (CAB)^+\ge \rk \left( AB (CAB)^+ CA\right)$,
in contradiction to the hypothesis.
\end{proof}

%
%
%

} 


\begin{thebibliography}{}

\bibitem{Aizerman64}
Mark~A. Aizerman, Emmanuel~M. Braverman, and Lev~I. Rozonoer.
\newblock Theoretical foundations of the potential function method in pattern
  recognition learning.
\newblock In \emph{Automation and Remote Control,}, number~25 in Automation and
  Remote Control,, pages 821--837, 1964.

\bibitem{BoserVapnik92}
Bernhard~E. Boser, Isabelle~M. Guyon, and Vladimir~N. Vapnik.
\newblock A training algorithm for optimal margin classifiers.
\newblock In \emph{Proceedings of the 5th Annual ACM Workshop on Computational
  Learning Theory}, pages 144--152. ACM Press, 1992.

\bibitem{LIBLINEAR}
R.-E. Fan, K.-W. Chang, C.-J. Hsieh, X.-R. Wang, and C.-J. Lin.
\newblock LIBLINEAR: A library for large linear classification.
\newblock \emph{Journal of Machine Learning Research} 9:\penalty0
1871--1874, 2008.

\bibitem{Scholkopf02}
Bernhard Sch{\"o}lkopf and Alexander~J Smola.
\newblock \emph{Learning with kernels}.
\newblock MIT Press, 2002.

\bibitem{KPCA1998}
Bernhard Sch{\"o}lkopf, Alexander Smola, and Klaus-Robert M{\"u}ller.
\newblock Nonlinear component analysis as a kernel eigenvalue problem.
\newblock \emph{Neural computation}, 10\penalty0 (5):\penalty0 1299--1319,
  1998.

\bibitem{Shawe-Taylor04}
John Shawe-Taylor and Nello Cristianini.
\newblock \emph{Kernel Methods for Pattern Analysis}.
\newblock Cambridge University Press, New York, 2004.

\bibitem{Vapnik95}
Vladimir N.\ Vapnik.
\newblock\emph{The Nature of Statistical Learning Theory}.
\newblock Springer Verlag, New York, 1995.

\bibitem{USPS}
USPS handwritten digits data set
\newblock \emph{A Database for Handwritten Text Recognition Research, J. J. Hull, IEEE PAMI 16(5) 550-554, 1994.}


\end{thebibliography}
\end{document}